\documentclass[conference,letterpaper]{article}

\usepackage[utf8]{inputenc} 
\usepackage[T1]{fontenc}
\usepackage{url}
\usepackage{ifthen}
\usepackage{cite}
\usepackage[cmex10]{amsmath} 

\usepackage{amsmath}
\usepackage{algorithm}
\usepackage[noend]{algpseudocode}

\usepackage{float}
\usepackage{graphicx}
\usepackage{amssymb}
\usepackage{epsfig}
\usepackage{amsmath}
\usepackage{amsthm}
\usepackage{amsfonts}
\usepackage{setspace}
\usepackage{url}
\usepackage{psfrag}
\usepackage{graphicx}
\usepackage{dsfont}
\usepackage{units}
\usepackage{graphicx}
\usepackage{wrapfig}

\newtheorem{thm}{Theorem}
\newtheorem{cor}{Corollary}
\newtheorem{lem}{Lemma}

\newcommand{\RomanNumeralCaps}[1]
{\MakeUppercase{\romannumeral #1}}

\usepackage{authblk}

\usepackage{enumitem}

        \author{Shlomi Vituri \\ School of Electrical Engineering \\Tel Aviv University \\Email: shlomivituri@mail.tau.ac.il 
        \and
        Meir Feder \\ School of Electrical Engineering \\Tel Aviv University \\Email: meir@tau.ac.il}

\begin{document}

	\title{Universal Batch Learning Under The Misspecification Setting}

	\maketitle

	
	
	\begin{abstract}		
	In this paper we consider the problem of universal {\em batch} learning in a misspecification setting with log-loss. In this setting the hypothesis class is a set of models $\Theta$. However, the data is generated by an unknown distribution that may not belong to this set but comes from a larger set of models $\Phi \supset \Theta$. Given a training sample, a universal learner is requested to predict a probability distribution for the next outcome and a log-loss is incurred. The universal learner performance is measured by the regret relative to the best hypothesis matching the data, chosen from $\Theta$.
    Utilizing the minimax theorem and information theoretical tools, we derive the optimal universal learner, a mixture over the set of the data generating distributions, and get a closed form expression for the min-max regret. We show that this regret can be considered as a constrained version of the conditional capacity between the data and its generating distributions set. We present tight bounds for this min-max regret, implying that the complexity of the problem is dominated by the richness of the hypotheses models $\Theta$ and not by the data generating distributions set $\Phi$. We develop an extension to the Arimoto-Blahut algorithm for numerical evaluation of the regret and its capacity achieving prior distribution. We demonstrate our results for the case where the observations come from a $K$-parameters multinomial distributions while the hypothesis class $\Theta$ is only a subset of this family of distributions.
	\end{abstract}
	
	\section{Introduction}\label{Intruduction}
	In the recent years there has been a flourishing interest in the field of statistical learning. In the common setting of the learning problem there is a batch of training samples from an unknown data source, used to  predict the next ``test'' outcome as accurately as possible, under a given loss function. This paper starts with discussing the unsupervised batch learning problem, which is essentially a probability estimation problem, where a batch of outcomes $y^{N-1}$ is given and the learner is requested to evaluate the probability for the next outcome $y_N$, see \cite{DistributionEstimation}. Using similar tools we extend our analysis to the supervised learning problem, see e.g., \cite{StatisticalLearning1},\cite{StatisticalLearning2}, in which the training is composed of pairs $x^{N-1},y^{N-1}$ of features and labels, and the goal is to predict the test label $y_N$ given the training and the test feature $x_N$.
    In contrast to the common batch setting in learning theory, the common setting in information theory is to predict sequentially, sample by sample in an online manner, the next outcome distribution of the data source. The resulting online prediction problem, in the unsupervised setting, essentially assigns a probability distribution to the whole data sequence $y^N$. 

    Universal learning or prediction also assumes a set of hypotheses $\Theta$. The non-universal learner who knows the data generation model can choose the best hypothesis in this set, which serves as a reference for the universal learner that does not know the true data generation distribution. Universal prediction with respect to a given class has been investigated thoroughly, see \cite{UniversalPrediction},\cite{RobustInference},\cite{SequentialMisspecification}, under various data generating distribution settings with log-loss measure. In addition, preliminary online supervised universal learning problems were also investigated in\cite{OnlineLearning}. All these works consider online learning. Recently, however, there was also a progress in formulating and solving the universal batch learning, both in supervised and unsupervised data settings, from the information theory point of view, see e.g., \cite{BatchLearning},\cite{UniversalLearningIndividual},\cite{UniversalLearningIndividualISIT},\cite{DeepPNML}. 

    In universal learning a crucial element is the assumption on the data generation mechanism. In the stochastic setting, it is assumed that the data is generated by some unknown distribution from the class $\Theta$ of reference models. On the other extreme, in the individual setting nothing is assumed on the data, it is an arbitrary individual sequence. Yet, there is another setting, the misspecified stochastic setting, where the data is generated from a distribution in a set of distributions $\Phi$, while the set of hypotheses $\Theta$, from which the reference learner is chosen, is only a subset of $\Phi$. This is actually the common case in ``agnostic'' statistical learning where the data generation mechanism is, say, any i.i.d. distribution, while the reference is a set of hypotheses $\Theta$. In information theoretic context, the first mentioning of this misspecified setting is in \cite{Barron}, but it was not investigated further until recently, where it was considered for online learning with log-loss, \cite{RobustInference},\cite{SequentialMisspecification} and for batch learning rates of convergence bounds in \cite{NirPaper}.
    
    This work takes the misspecified setting a step further to the case of universal batch learning. It consider first the unsupervised case, and given a training batch $y^{N-1}$ of size $N-1$, it finds a universal distribution for the next outcome $y_N$, denoted by $Q(y_N|y^{N-1})$. The work analyzes the regret associated with the universal learner, which is expectation with respect to the true distribution of the difference between the log-loss associated with the universal predictor and that of the hypothesis from the set $\Theta$, that best fits the data. The extension to the supervised setting, where the data features $x^N$ are i.i.d, is quite straightforward.
    
    Indeed, our main contributions in this paper is the exact formulation of the min-max regret and its corresponding universal probability for the ``test'' outcome given a batch of data samples and under the misspecification setting. We show that the min-max regret can be thought of as a constrained version of the conditional capacity between the data and the set of data generating distributions $\Phi$. Interestingly, due to the constraint, we show that the min-max regret turns out to be approximately equals to the regret in the stochastic case assuming that the data was generated by a distribution from the set of hypotheses $\Theta$ rather than the larger set $\Phi$. This implies that the complexity of the problem is actually a function of the hypotheses set rather than the larger family of distributions that may generate the data. 

    As we have mentioned above the regret can be interpreted as a constraint version of the conditional capacity between $Y_N$ and $\Phi$.
    In general, it is hard to evaluate an analytical closed form for the capacity achieving prior distribution and for the capacity. In classical communication theory and recently also in universal theory theory, various of of extensions to the classical Arimoto-Blahut algorithms \cite{Blahut}\cite{Arimoto} were developed for numerical evaluation of the capacity achieving prior distribution and its resultant capacity \cite{FogelFederCalculation}\cite{ArimotoBlahutExt1}\cite{ArimotoBlahutExt2}\cite{ArimotoBlahutExt3}\cite{ArimotoBlahutExt4}\cite{RobustInference}. Therefore, for numerical evaluation of capacity achieving prior distribution and the regret we developed an extension for the Arimoto-Blahut algorithm for the universal batch prediction under the misspecification setting, and demonstrates the theoretical results by numerical evaluation over the i.i.d Bernoulli distributions sets.

    The paper outline is as follows: in section \ref{ProblemStatement} we give a formal problem statement and basic definitions for the universal batch learning under the misspecification setting. Following that, in section \ref{Misspecified Batch Learning Analysis} we analyze and present our main contributions: min-max regret derivation, bounds for the regret and their meaning, and examples. In section \ref{Arimoto Blahut Algorithm Extension} we develop the Arimoto-Blahut algorithm extension and in section \ref{Misspecified Universal Batch Learning Extensions} we extend our main results to the combined batch and online learning and to the supervised setting. In section \ref{Conclusions and Future Research Directions} we finally conclude our work and present future possible directions for further research.

	\section{Problem Statement}\label{ProblemStatement}
	\subsection{Learning under the misspecification setting}
	Let us define the problem of batch learning under the misspecification setting as follows: First, $N$ data samples, denoted by $y^{N}$, are generated by a given distribution $P_\phi(y^{N})$ from a set of potential data generating distributions $\Phi$. The learner gets the batch of training samples $y^{N-1}$, and predicts the next outcome's conditional distribution, denoted by $Q(y_N|y^{N-1})$, in a universal manner, i.e.,  without any knowledge of the specific selected data generating distribution. The accuracy of the universal learner is measured by the well-known log-loss function. The performance of the universal learner is evaluated by the regret, i.e., relative to the log-loss of the best learner from a set of hypotheses or distributions, denoted by $\Theta$, given the batch of training samples, where each hypothesis is defined by a probability distribution $P_\theta(y^N)$. Loosely speaking, the misspecification setting means that there is a mismatch between the data generating distributions set $\Phi$ and a smaller set of hypotheses $\Theta \subset \Phi$ that tries to predict the data generating distribution. 
	\subsection{Basic definitions}
	To formulate and analyze the the problem we will use the following definitions.
 
	The regret in case of data generating distribution $P_\phi \in \Phi$, hypothesis $P_\theta \in \Theta$ and a predictive distribution $Q$ is defined by:
	\begin{equation}
		R_N(P_{\phi},P_{\theta},Q)=\sum_{y^N}P_{\phi}(y^N)\log \left({\frac{P_{\theta}(y_N|y^{N-1})}{Q(y_N|y^{N-1})}}\right)
	\end{equation}
	In this paper we are interested in the analysis of the min-max regret which is given by the following:
	\begin{equation}
            \label{minmaxRegretDefinition}
		R^*_N(\Theta,\Phi) = \min_{Q} \max_{P_{\phi} \in \Phi} \max_{P_{\theta} \in \Theta} R_N(P_{\phi},P_{\theta},Q)
	\end{equation}
	In words, the min-max regret minimizes the regret of the worst data generating distribution and best matching hypothesis to the data, via a universal distribution $Q$, given the batch training. 	
	Let us define also the conditional capacity of the distributions set $\Phi$ by the following:	
	\begin{align}
		\begin{aligned}
			C_{c,N}(\Phi) &\equiv \max_{\pi({\phi})}I(Y_N;\Phi|Y^{N-1})
                \\& = \int_{\phi} \sum_{y^N} P(y^N,\phi)       \log \left({\frac{P(y_N|y^{N-1},\phi)}{P({y_{N}|y^{N-1})}}}\right)        \,d\phi  
		\end{aligned}
	\end{align}
	where $\pi(\phi)$ is the conditional capacity prior distribution and $P(Y^N=y^N,\Phi=\phi) \equiv \pi(\phi)P_\phi(y^N)$.
        
        The conditional KL divergence will be defined by:
	\begin{align}
		\begin{aligned}
            \label{condDivergence}
			D_{c,N}(P_{\phi}||P_{\theta}) &\equiv D(P_{\phi}(Y_N|Y^{N-1})||P_{\theta}(Y_N|Y^{N-1}))
            \\& = \sum_{y^N} P_\phi(y^N) \log\left(\frac{P_{\phi}(y_N|y^{N-1})}{P_{\theta}(y_N|y^{N-1})}\right)
		\end{aligned}
	\end{align}
	and the conditional divergence between $P_{\phi} \in \Phi$ and a set of distributions $\Theta$ by:
	\begin{align}
		\begin{aligned}
            \label{DivergenceWithSet}
			D_{c,N}(P_{\phi}\|\Theta) \equiv \inf_{{P_{\theta}\in\Theta}}D_{c,N}(P_{\phi}\|P_{\theta}).
		\end{aligned}
	\end{align}
    
    In addition, let us define 
    the expectation of any arbitrary function $f(\phi)$, w.r.t the probability distribution $\pi_(\phi)$ by the following operator:
    \begin{align}
        \begin{aligned}
        \label{PiExpectation}
            E_{\pi(\phi)} \{  f(\phi) \} = \int_{\phi} \pi{(\phi)} f(\phi) \,d\phi.
        \end{aligned}
    \end{align}        
    Note that for compactness, with small abuse of notation, we will occasionally denote the operator $E_{\pi(\phi)} \{ \cdot \}$ by $E_{\pi} \{  \cdot \}$.
    
    These definitions will be used in the analysis and derivations of our main results.		
    
        \section{Misspecified Batch Learning Analysis}\label{Misspecified Batch Learning Analysis}
        
        \subsection{Min-Max Regret Derivation}
	Our first main result for the problem of universal batch learning under the misspecification setting, is an analytical formulation for the min-max regret, given as follows:
	\begin{thm}\label{ThmFirst}
		The min-max regret of the problem defined in \RomanNumeralCaps{2}-A is given by:		
		\begin{align}									
			\begin{aligned}		
				\label{thmRegret}	
			R^*_N(\Theta,\Phi) = \max_{\pi({\phi})}  \left( I(Y_N;\Phi|Y^{N-1}) - E_{\pi(\phi)} \{  D_{c,N}({{P_{\phi}}\|\Theta}) \} \right)
		\end{aligned}
		\end{align}
		and the universal distribution for a given $\pi(\phi)$ is given by:
		\begin{align}
			\begin{aligned}
				Q_{\pi}({y_N|y^{N-1}}) = \frac{\int_{\phi}\pi(\phi)P_{\phi}(y^{N})  \,d\phi}{\int_{\phi}\pi(\phi)P_{\phi}(y^{N-1})  \,d\phi}.
			\end{aligned}
		\end{align}		
	\end{thm}
	
	\begin{proof}
		The regret is given by the following:
		\begin{equation}
			R_N(P_{\phi},P_{\theta},Q)=\sum_{y^N}P_{\phi}(y^N)\log \left({\frac{P_{\theta}(y_N|y^{N-1})}{Q(y_N|y^{N-1})}}\right)
		\end{equation}
		By simple algebraic manipulations:
		\begin{align}
			\begin{aligned}
			R_N(P_{\phi},P_{\theta},Q) & =\sum_{y^N}P_{\phi}(y^N)\log \left({\frac{P_{\phi}(y_N|y^{N-1})}{Q(y_N|y^{N-1})}}\right) \\& - \sum_{y^N}P_{\phi}(y^N)\log \left({\frac{P_{\phi}(y_N|y^{N-1})}{P_{\theta}(y_N|y^{N-1})}}\right)
		\end{aligned}
		\end{align}
		Using the conditional KL divergence, as defined in (\ref{condDivergence}), we are getting the below formulation:
		\begin{align}
			\begin{aligned}
			R_N(P_{\phi},P_{\theta},Q) = D_{c,N}({{P_{\phi}}\|{Q}}) - D_{c,N}({{P_{\phi}}\|P_{\theta}})
			\end{aligned}
		\end{align}
		Hence, the optimal regret in the sense of min-max regret, as defined in (\ref{minmaxRegretDefinition}), can be rewritten by the following:
		\begin{align}
			\begin{aligned}
                \label{minmaxRegretInKL}
			R^*_N(\Theta,\Phi) = \min_{Q} \max_{P_{\phi} \in \Phi} \max_{P_{\theta} \in \Theta} \left(D_{c,N}({{P_{\phi}}\|{Q}}) - D_{c,N}({{P_{\phi}}\|P_{\theta}})\right)
			\end{aligned}
		\end{align}
		Since the first term in (\ref{minmaxRegretInKL}) is not a function of $\Theta$, we get the following:
		\begin{align}
			\begin{aligned}
                \label{minmaxRegretAsDiffOfDivergences}
			R^*_N(\Theta,\Phi) & = \min_{Q} \max_{P_{\phi} \in \Phi} ( D_{c,N}({{P_{\phi}}\|{Q}}) - \min_{P_{\theta} \in \Theta} D_{c,N}({{P_{\phi}}\|P_{\theta}}) )
			\\& = \min_{Q} \max_{P_{\phi} \in \Phi} \left( D_{c,N}({{P_{\phi}}\|{Q}}) - D_{c,N}({{P_{\phi}}\|\Theta}) \right)
			\end{aligned}
		\end{align}
        where we used definition (\ref{DivergenceWithSet}) in the last step of equation (\ref{minmaxRegretAsDiffOfDivergences}).
		
        Let us translate the min-max problem into a mixture min-max problem by the following:
        \begin{align}
			\begin{aligned}
                \label{mixtureRegret}
			R^*_N(&\Theta,\Phi) = \min_{Q} \max_{\pi({\phi})} 
                \int_{\phi} \left( D_{c,N}({{P_{\phi}}\|{Q}}) - D_{c,N}({{P_{\phi}}\|\Theta}) \right) \,d\phi
                \\&= \min_{Q} \max_{\pi({\phi})} 
                \left( E_\pi \{  D_{c,N}({{P_{\phi}}\|{Q}}) \} 
			- E_\pi \{  D_{c,N}({{P_{\phi}}\|{\Theta}}) \} \right)
			\\& \equiv \min_{Q} \max_{\pi({\phi})} R_N(\pi(\phi),Q)
			\end{aligned}
		\end{align}
        where we used, for simplicity of presentation, definition (\ref{PiExpectation}) in (\ref{mixtureRegret}). Note that by definition:
		\begin{align}
			\label{TwoTermsRegret}
			\begin{aligned}
				R_N(\pi(\phi),Q) \equiv
				\underbrace{E_\pi \{  D_{c,N}({{P_{\phi}}\|{Q}}) \}}_{\propto \pi(\phi);\propto -\log{Q}} -
				\underbrace{E_\pi \{  D_{c,N}({{P_{\phi}}\|{\Theta}}) \}}_{\propto \pi(\phi);\text{not a function of } Q}
			\end{aligned}
		\end{align}
  
		  Since the first term of $R_N(\pi(\phi),Q)$ is proportional to $-\log(Q)$ and the second term is not a function of $Q$, (\ref{TwoTermsRegret}) is a convex function w.r.t $Q$. 
		Moreover, both the first and second terms are concave functions w.r.t $\pi(\phi)$, due to their linearity in $\pi(\phi)$. Thus, $R_N(\pi(\phi),Q)$ is also a concave function w.r.t $\pi(\phi)$.
		Hence, according to Sion's minimax Theorem \cite{Sion} for convex-concave functions, the min-max problem can be translated into a max-min problem as follows:
  
		\begin{align}
			\begin{aligned}
				\label{maxminRegret}
				R^*_N(\Theta,\Phi) = \max_{\pi({\phi})} \min_{Q} R_N(\pi(\phi),Q).
			\end{aligned}
		\end{align}	
		Now let us minimize $R_N(\pi(\phi),Q)$ w.r.t $Q$ by zeroing the derivative of the following Lagrangian:
		\begin{align}
			\begin{aligned}
			L = R_N(\pi(\phi),Q) + \sum_{y^{N-1}}\lambda_{y^{N-1}}\sum_{y_N}Q(y_N|y^{N-1})
			\end{aligned}
		\end{align}	
		\begin{align}
			\begin{aligned}
			\frac{\partial L}{\partial Q} \Big{|}_{Q_{\pi}} = - \frac{1}{Q_{\pi}(y_N|y^{N-1})}\int_{\phi}\pi(\phi)P_{\phi}(y^{N})  \,d\phi + \lambda_{y^{N-1}} = 0
			\end{aligned}
		\end{align}	
		Therefore, we get:
		\begin{equation}
			Q_{\pi}({y_N|y^{N-1}}) = \frac{\int_{\phi}\pi(\phi)P_{\phi}(y^{N})  \,d\phi}{\lambda_{y^{N-1}}}
		\end{equation}	
		and in order to meet the constraint of $\sum_{y_{N}}Q_{\pi}({y_N|y^{N-1}}) = 1$ we get:
		\begin{align}
			\begin{aligned}
			\sum_{y_N}Q_{\pi}({y_N|y^{N-1}}) & = \frac{\int_{\phi}\pi(\phi)\sum_{y_{N}}P_{\phi}(y^{N})  \,d\phi}{\lambda_{y^{N-1}}} \\& = 
			\frac{\int_{\phi}\pi(\phi)P_{\phi}(y^{N-1})  \,d\phi}{\lambda_{y^{N-1}}} = 1
			\end{aligned}
		\end{align}	
		which leads immediately to the following Lagrange multiplier:
		\begin{equation}
			\lambda_{y^{N-1}} = \int_{\phi}\pi(\phi)P_{\phi}(y^{N-1})  \,d\phi.
		\end{equation}	
		Combining all the above, we get the minimizer $Q$ as:
		\begin{equation}
                \label{condQ}
			Q_{\pi}({y_N|y^{N-1}}) = \frac{\int_{\phi}\pi(\phi)P_{\phi}(y^{N})  \,d\phi}{\int_{\phi}\pi(\phi)P_{\phi}(y^{N-1})  \,d\phi}
		\end{equation}		
		or in other words,
		\begin{equation}
                \label{directQ}
			Q_{\pi}({y^{N}}) = \int_{\phi}\pi(\phi)P_{\phi}(y^{N})  \,d\phi,
		\end{equation}
            where similarly to the online case, (\ref{directQ})  is a mixture distribution over the set $\Phi$, see for example \cite{UniversalPrediction} and \cite{RobustInference}. As will be discussed later, the choice of the mixture prior is making the difference.
		
            Using (\ref{condQ}) We observe that:
		\begin{align}
			\begin{aligned}
			\label{FirstTerm}
			&E_\pi \{  D_{c,N}({{P_{\phi}}\|{Q_\pi}}) \} = \int_{\phi} \pi{(\phi)} D_{c,N}({{P_{\phi}}\|{Q_\pi}}) \,d\phi 
			\\& = \int_{\phi} \sum_{y^N} \pi{(\phi)}P_{\phi}(y^N)       \log \left({\frac{P_{\phi}(y_N|y^{N-1})}{Q_{\pi}(y_N|y^{N-1})}}\right)        \,d\phi
			\\& = \int_{\phi} \sum_{y^N} P(y^N,\phi)       \log \left({\frac{P(y_N|y^{N-1},\phi)}{P({y^{N})/P(y^{N-1})}}}\right)        \,d\phi  
			\\&  = E_{P(Y^N,\Phi)}  \left\{     \log \left({\frac{P(Y_N|Y^{N-1},\Phi)}{P({Y_{N}|Y^{N-1})}}}\right)  \right\} = I(Y_N;\Phi|Y^{N-1})
			\end{aligned}
		\end{align}
		where $P(Y^N=y^N,\Phi=\phi) \equiv \pi(\phi)P_\phi(y^N)$. Combining (\ref{TwoTermsRegret}), (\ref{maxminRegret}) and (\ref{FirstTerm}) we get (\ref{thmRegret}).
	\end{proof}    

    Theorem \ref{ThmFirst} shows that the batch learning min-max regret, under the misspecification setting, can be regarded as a constrained version of the conditional capacity, i.e., a constrained version of the batch learning regret under the classical stochastic setting \cite{BatchLearning}. 
    Moreover, the form of the result is similar to the online learning min-max regret, under the misspecification setting, as shown in \cite{RobustInference} Theorem 2 and in \cite{SequentialMisspecification}. In both cases the form of the regret is a combination of two terms: one is mutual information between the samples and the data source of distributions $\Phi$, given by $I(Y_N;\Phi|Y^{N-1})$ in our problem, and an additional penalty term, denoted by $E_{\pi(\phi)} \{  D_{c,N}({{P_{\phi}}\|\Theta}) \}$ in our problem, which quantifies the mismatch between the set of hypotheses $\Theta$ and the data generating distributions set $\Phi$, by the closest projected distribution from the set $\Phi$ into the set of hypotheses $\Theta$, in the sense of KL divergence.
    
    One might mistakenly think that the regret equals to the conditional capacity between the data samples and the data generating distributions family $\Phi$, but since the optimal prior distribution $\pi(\phi)$ should maximize the difference between these two terms this is not the case. For example, let us consider the following extreme: $\Phi$ is the $K$-parameters multinomial distributions set, corresponding to a distribution over an alphabet of size $K+1$, while $\Theta$ is only the $K'$-parameters multinomial distributions over an alphabet of size $K'+1$, where $K'<K$. In this case, $D_{c,N}(P_{\phi}\|P_{\theta}) = 0$ when $P_{\phi} \in \Theta$, however, when $P_{\phi} \notin \Theta$ a zero probability is assigned to symbols that are not in $\Theta$ and so $D_{c,N}(P_{\phi}\|\Theta)=\infty$. Hence, clearly the prior distribution $\pi(\phi)$ must be zeroed outside the set of hypotheses $\Theta$, i.e., $\pi(\phi) = 0 ~ \forall \phi \notin \Theta$, in order to maximize (\ref{thmRegret}). In other words, in this example, the mixture distribution should be taken only over the hypotheses set $\Theta$, and therefore the regret equals exactly to the conditional capacity of the set $\Theta$, which according to \cite{parametricExample} equals to $C_{c,N}(\Theta) = \frac{K'}{2N} + o(n^{-1})$, as in the batch learning under the stochastic setting in which $\Phi\equiv\Theta$.
    Moreover, even when $D_{c,N}(P_{\phi}\|\Theta)<\infty$, surprisingly, we will show later in Theorem \ref{Thm2} that due to this constraint the prior $\pi(\phi)$ that maximizes (\ref{thmRegret}) is such that most of its distribution mass is concentrated over $\Theta$, and not over the entire set of data generating distributions $\Phi$. Thus, the complexity of the problem is governed by the set of hypotheses $\Theta$, rather than the larger set of data generating distributions $\Phi$.
    
    In addition, Theorem \ref{ThmFirst} shows that the universal distribution $Q$ is a mixture over the distributions set, as in all of the known unsupervised universal learning problems: both in the online learning under the classical stochastic and individual sequences settings \cite{UniversalPrediction}, the misspecification setting \cite{RobustInference}\cite{SequentialMisspecification}, and the batch learning under the stochastic setting \cite{BatchLearning}.
    \subsection{Bounding the Min-Max Regret}
	Our second main contribution, bounds the min-max regret from the below simply by the conditional capacity of the set of distributions $\Theta$. 
    Moreover, it bounds the regret from the above approximately by the conditional capacity of a small extension of that set $\Theta$. 
    The proof of this upper bound requires the following Lemma, that upper bounds $J(\pi) \equiv I(Y_N;\Phi|Y^{N-1}) - E_{\pi}\{D_{c,N}(P_\phi\|\Theta)\}$. The proof here is similar to that of Lemma 11 in \cite{SequentialMisspecification}.
 
	\begin{lem}
		\label{lemma1}
		For any $\pi_{0}(\phi)$, $\pi_{1}(\phi)$ and $\lambda \in [0,1]$ 
		\begin{align}
			\begin{aligned}
				J(\lambda \pi_1 + (1-\lambda) \pi_{0}) \leq \lambda J(\pi_1) + (1-\lambda) J(\pi_{0}) + h(\lambda)
			\end{aligned}
		\end{align}
		where $h(\lambda)$ is the binary entropy of a Bernoulli distribution $Ber(\lambda)$.
	\end{lem}	
	\begin{proof}
		Let us define the Markov chain triplet $B \to \Phi \to Y^N$, where $B \sim Ber(\lambda)$, $\lambda \in [0,1]$, $\Phi = \phi$ is conditionally distributed according to $\pi_{b}(\phi)$ given $B=b$ and $Y^N=y^N$ is conditionally distributed according to $P_{\phi}(y^N)$ given $\phi$.
		Note that the induced prior distribution over $\Phi$ is given by $\pi(\phi) = \lambda \pi_{1}(\phi) + (1-\lambda) \pi_{0}(\phi)$,
		and by definition we get for any $b \in \{0,1\}$:
		\begin{align}			
			\begin{aligned}				
				\label{lem1}
				J(\pi_{b}|B=b) & = I(Y_N;\Phi|Y^{N-1},B=b) 
				\\& - E_{\pi_{b}}\{ D_{c,N} (P_{\phi} \| \Theta ) | B=b \}
			\end{aligned}
		\end{align}
		By taking the expectation of (\ref{lem1}) according to the $B$ distribution we get:
		\begin{align}
			\begin{aligned}
				\label{lem2}
				E_B\{ J(\pi_{b}|B=b) \} = \lambda J(\pi_{1}) + (1-\lambda) J(\pi_{0})				
			\end{aligned}
		\end{align}
		while on the other hand,
		\begin{align}
			\begin{aligned}
				\label{lem3}
				E_B\{ J(\pi_{b}|B=b) \} &= I(Y_N;\Phi | B, Y^{N-1}) 
				\\& - E_{\pi}\{ D_{c,N} (P_{\phi} \| \Theta ) \}
			\end{aligned}
		\end{align}
		In addition, due to the Markov chain characteristics, we have $B\perp Y^N|\Phi$, which leads to $I(Y_N;B|\Phi,Y^{N-1}) = 0$. Therefore, by applying twice the mutual information chain rule we get:
		\begin{align}
			\begin{aligned}
				\label{lem4}
				I(Y_N;\Phi|Y^{N-1}) &= I(Y_N;B,\Phi|Y^{N-1}) - I(Y_N;B|\Phi,Y^{N-1})				
				\\& = I(Y_N;B,\Phi|Y^{N-1})
				\\& = I(Y_N;\Phi|B,Y^{N-1}) + I(Y_N;B|Y^{N-1}) 
				\\&\leq I(Y_N;\Phi|B,Y^{N-1}) + h(\lambda) 
			\end{aligned}
		\end{align}		
		Combining (\ref{lem1}), (\ref{lem2}), (\ref{lem3}) and (\ref{lem4}) with $J(\pi) = I(Y_N;\Phi|Y^{N-1}) - E_{\pi}\{ D_{c,N} (P_{\phi} \| \Theta ) \}$ completes the proof.
	\end{proof}

 We are now ready to state the following Theorem which bounds the min-max regret in terms of the conditional capacities of the set $\Theta$ and a slightly larger set:
	\begin{thm}\label{Thm2}
		Suppose $\Phi$ and $\Theta$ are sets of distributions s.t. $C_{c,N}(\Phi) = \tau_N \to 0$ and $\Theta \subseteq \Phi$. Then for every $\epsilon_N \gg \tau_N$ we have
		\begin{align}
			\begin{aligned}
                \label{RegretBounds}
				C_{c,N}(\Theta) \leq R^*_N(\Theta,\Phi) \leq C_{c,N}(\Theta_{{\epsilon}_N}) + o(1)
			\end{aligned}
		\end{align}		
		where $\Theta_{\epsilon} \equiv \{P_{\phi} \in \Phi: D_{c,N}(P_\phi||\Theta) < \epsilon \}$.
	\end{thm}
	
	\begin{proof}
		Since $R^*_N(\Theta,\Phi) = \max_{\pi({\phi})} J(\pi) \geq 0$ there exists a $\pi(\phi)$ such that
		\begin{align}
			\begin{aligned}
				0 \leq J(\pi) & = I(Y_N;\Phi|Y^{N-1}) - E_\pi\{D_{c,N}(P_\phi \| \Theta)\} \\& \leq C_{c,N}(\Phi) -  E_\pi\{D_{c,N}(P_\phi \| \Theta)\}
			\end{aligned}
		\end{align}

		Therefore, the fraction of models of distributions, denoted by $\lambda$, from the set $\Phi$ that are not included in $\Theta_{\epsilon}$ can be upper bounded by Markov's inequality as follows:
		\begin{align}
			\begin{aligned}
				\lambda = P(D_{c,N}(P_\phi \| \Theta) > \epsilon) \leq \frac{E_\pi\{D_{c,N}(P_\phi \| \Theta)\}}{\epsilon} \leq \frac{C_{c,N}(\Phi)}{\epsilon}
			\end{aligned}
		\end{align}		
	Let us now define $\pi_{0}(\phi),\pi_{1}(\phi)$ as the distributions implied by $\pi(\phi)$ over the sets $\Theta_{\epsilon}$ and its complement, i.e., $\pi_{0}(\phi)=\pi(\phi)/(1-\lambda),\;\phi\in \Theta_\epsilon$ and zero otherwise, while $\pi_{1}(\phi)=\pi(\phi)/\lambda,\;\phi\notin \Theta_\epsilon$ and zero otherwise. As a consequence, $\pi(\phi) = \lambda \pi_{1}(\phi) + (1-\lambda) \pi_{0}(\phi),\;\forall \phi\in \Phi$.	
		Applying Lemma \ref{lemma1} and maximizing over $\pi_{0}(\phi)$ and $\pi_{1}(\phi)$ gives us the following:
		\begin{align}
			\label{SecResult1}
			\begin{aligned}
				J(\pi) &\leq \lambda J(\pi_1) + (1-\lambda) J(\pi_0) + h(\lambda)
				\\& \leq \lambda R^*_N(\Theta,\Phi) + (1-\lambda) R^*_N(\Theta,\Theta_{\epsilon}) + h(\lambda).
			\end{aligned}
		\end{align}		
		Maximizing (\ref{SecResult1}) over $\pi(\phi)$, combined with simple algebraic manipulations gives us the following:
		\begin{align}
			\begin{aligned}
				R^*_N(\Theta,\Phi) \leq R^*_N(\Theta,\Theta_{\epsilon}) + \underbrace{ \frac{h(\lambda)}{1-\lambda} }_{A(\lambda)}.
			\end{aligned}
		\end{align}
		By taking $\epsilon = \epsilon_N \gg \tau_N$, we get $\lambda = \lambda_N \to 0$, which leads to $A(\lambda_N) \to 0$. 
		Furthermore, since $\Theta \subseteq \Theta_{\epsilon}$ for any $\epsilon \geq 0$, then clearly $R^*_N(\Theta,\Theta_{\epsilon}) \leq C_{c,N}(\Theta_{\epsilon})$.		
		Hence, we get:
		\begin{align}
			\begin{aligned}
				C_{c,N}(\Theta) \leq R^*_N(\Theta,\Phi) \leq C_{c,N}(\Theta_{\epsilon_N}) + o(1)
			\end{aligned}
		\end{align}
		where the lower bound is given by the definition of the regret.
	\end{proof}

    Theorem \ref{Thm2} provides for the batch learning case a similar result to what was shown for the online learning under misspecification in \cite{SequentialMisspecification} Theorem 4. 
    
	Interestingly, it should be noted that while $\epsilon_N \gg \tau_N = C_{c,N}(\Phi)$, it does not mean that the conditional capacity of $\Theta_{\epsilon_N}$ is greater than the conditional capacity of $\Phi$. Its true meaning is that only a shell extension to the set $\Theta$, quantified by $\epsilon_N$, of distributions from $\Phi$ affects the min-max regret performance. If in addition $\epsilon_N \to 0$, the set $\Theta_{\epsilon_N}$ is a small extension of $\Theta$, implying that the resulting conditional capacities of $\Theta$ and $\Theta_{\epsilon_N}$ might be close. An illustration of this phenomenon is given in Figure \ref{FigIllustration}.
 
    In many interesting examples we indeed have $\epsilon_N \to 0$ and the conditional capacities coincide for large $N$. Such an example is where the observations come from a distribution in the set $\Phi$ of $K$-parameters multinomial distributions of the form: $(\phi_0,\phi_1,\dots,\phi_K)$, s.t $\sum_{k=0}^{K} \phi_k=1$. In \cite{parametricExample} it was shown that the regret, which is the conditional capacity of $\Phi$ in the stochastic setting of batch learning where the hypotheses class is also $\Phi$, equals to $C_{c,N}(\Phi) = \frac{K}{2N} + o(n^{-1})$. Therefore, by choosing $\epsilon_N = \frac{1}{N^{1-\alpha}}$, for any $0<\alpha<1$, we have both $\epsilon_N \gg C_{c,N}(\Phi)$ and $\epsilon_N \to 0$.    
    As noted, this may imply that $C_{c,N}(\Theta_{\epsilon_N}) \to C_{c,N}(\Theta)$ and according to a sandwich argument and (\ref{RegretBounds}), the min-max regret tends to $R^*_N(\Theta,\Phi) \to C_{c,N}(\Theta)$, which is the min-max regret for the batch learning under the classical stochastic setting, \cite{BatchLearning}. 
    
    To demonstrate this phenomenon, we can choose as an example, $K=1$, i,e., the Bernoulli distribution $Ber(\phi)$, where $\phi\in[0,1]$, to be the set of all the data generating distributions $\Phi$, and the set of hypotheses $\Theta$, to be the set of all $Ber(\theta)$, where $\theta\in[a,b]$ s.t. $0 \leq a < b \leq 1$. In this case, assuming the data samples are i.i.d and $\epsilon$ is small enough, it can be verified that $\Theta_{\epsilon}$ is the set of all $Ber(\theta_\epsilon)$, where $\theta_\epsilon \in\left[a- \delta_\epsilon(a),b+\delta_\epsilon(b)\right]$ and $\delta_\epsilon(c) = \sqrt{2{c(1-c){\epsilon}}}$, i.e., a small extension of the set $\Theta$. Note that in this example the  conditional capacity is a continuous function of $\epsilon$ (it is a composition of elementary functions). Thus, by taking $\epsilon = \epsilon_N$ as explained above, we get $C_{c,N}(\Theta_{\epsilon_N}) \to C_{c,N}(\Theta)$, and finally the min-max regret is approximately equal to $R^*_N(\Theta,\Phi) \approx C_{c,N}(\Theta)$.
    
    This observation means that under the conditions specified, the min-max regret in the misspecification setting converges to the regret in the case where the data generating distributions are approximately the ``smaller'' set of hypotheses $\Theta$ and not the ``bigger'' set $\Phi$. Another consequence is the fact that the universal distribution $Q$, is approximately a mixture distribution 
    where the prior conditional capacity achieving distribution, $\pi(\phi)$, is concentrated mostly over the set $\Theta$.			
    Note that $\epsilon_N \to 0$ does not always mean that $C_{c,N}(\Theta_{\epsilon_N}) \to C_{c,N}(\Theta)$, but in general such cases are esoteric. Such an esoteric example was given in \cite{SequentialMisspecification} Appendix F for the online setting. 

    \begin{figure}[h]        
        \centering
        \includegraphics[width=0.5\textwidth]{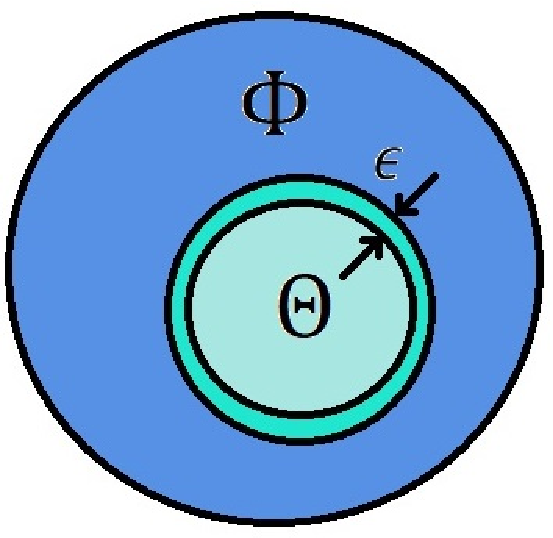}
        \caption{Theorem \ref{Thm2}} illustration
        \label{FigIllustration}
    \end{figure}
     
    \section{Arimoto Blahut Algorithm Extension}\label{Arimoto Blahut Algorithm Extension}
    As we have shown the regret can be interpreted as a constrained version of the conditional capacity between $Y_N$ and $\Phi$. Hence, $\pi(\phi)$ can be interpreted as the capacity achieving prior distribution.
    In general, it is hard to evaluate an analytical closed form for the capacity achieving prior distribution and for the capacity. In classical communication theory and recently also in universal prediction theory, various extensions to the classical Arimoto-Blahut algorithms \cite{Blahut}\cite{Arimoto} were developed for numerical evaluation of the capacity achieving prior distribution and its resultant capacity \cite{FogelFederCalculation}\cite{ArimotoBlahutExt1}\cite{ArimotoBlahutExt2}\cite{ArimotoBlahutExt3}\cite{ArimotoBlahutExt4}\cite{RobustInference}. Therefore, for numerical evaluation of $\pi(\phi)$ and the regret $R_N^*(\Theta,\Phi)$ we developed an extension for the Arimoto-Blahut algorithm for the universal batch prediction under the misspecification setting. The following Theorem \ref{thmArimotoBlahut} derives an upper and lower bounds to the regret. These bounds are used as convergence criteria of the iterative Arimoto-Blahut algorithm extension. Moreover, the form of the bounds as a difference between two conditional divergences implies the structure of the algorithm as we will show in the sequel.
    \subsection{Algorithm Development}
    The following Theorem derives an upper and lower bounds to the regret and implies the structure of the Arimoto-Blahut extension algorithm.
    \begin{thm}	    
        \label{thmArimotoBlahut}
        The min-max regret of the batch learning under the misspecification setting holds the following for any $\Phi$, $\Theta$ and $\pi(\phi)$:		
        \begin{align}									
            \begin{aligned}	
                \label{ArimotoBlahutRegret}
                R_L \leq R^*_N(\Theta,\Phi) \leq R_U 
              \end{aligned}
        \end{align}           
        where,
        \begin{align}
            \begin{aligned}
                R_L \equiv E_{\pi(\phi)}\left\{ D_{c,N}(P_\phi \| Q_{\pi}) -  D_{c,N}(P_\phi \| \Theta) \right\}
            \end{aligned}
        \end{align}
        and 
        \begin{align}
            \begin{aligned}
                R_U \equiv \max_{P_{\phi}}\left( D_{c,N}(P_\phi \| Q_{\pi}) -  D_{c,N}(P_\phi 
        \| \Theta) \right).
            \end{aligned}
        \end{align}
    \end{thm}
    \begin{proof}
    The regret can be written by the following:
    \begin{align}
        R^*_N(\Theta,\Phi) = \max_{\pi(\phi)}\left( E_{\pi(\phi)}\left\{ D_{c,N}(P_\phi \| Q_{\pi}) -  D_{c,N}(P_\phi \| \Theta) \right\} \right)
    \end{align}
    Let us denote by $\pi^*(\phi)$ the maximizing prior distribution of the regret. Therefore, $\forall \pi(\phi) \neq \pi^*(\phi)$ we get the following lower bound:
    \begin{align}
        \begin{aligned}                    
        R^*_N(\Theta,\Phi) & = E_{\pi^*(\phi)}\left\{ D_{c,N}(P_\phi \| Q_{\pi^*}) -  D_{c,N}(P_\phi \| \Theta) \right\} 
        \\& \geq E_{\pi(\phi)}\left\{ D_{c,N}(P_\phi \| Q_{\pi}) -  D_{c,N}(P_\phi \| \Theta) \right\} \end{aligned}         
    \end{align}                    
    On the other hand, by definition, we get the following upper bound:
    \begin{align}
        \begin{aligned}                    
        R^*_N(\Theta,\Phi) & = \min_{Q} \max_{P_{\phi} \in \Phi} \max_{P_{\theta} \in \Theta} R_N(P_{\phi},P_{\theta},Q) 
        \\& \leq \max_{P_{\phi} \in \Phi} \max_{P_{\theta} \in \Theta} R_N(P_{\phi},P_{\theta},Q_{\pi}) 
        \\& = \max_{P_{\phi}}\left( D_{c,N}(P_\phi \| Q_{\pi}) -  D_{c,N}(P_\phi \| \Theta) \right)
        \end{aligned}         
    \end{align}
    \end{proof}
    
    Using Theorem \ref{thmArimotoBlahut} we can extend the Arimoto-Blahut algorithm to the misspecified batch learning setting by the following \textbf{Algorithm} \textbf{\ref{ArimotoBlahutAlg}}. The inputs to the algorithm are the batch size $N$, an optimization parametr $\lambda$, a required convergence accuracy parameter $\epsilon$, the sets $\Phi$ and $\Theta$ and an initial prior distribution $\pi_0(\phi)$ (uniform distribution over the set $\Phi$ is a good common practice initialization). In the initialization stage we calculate the lower and upper bounds $R_L$ and $R_U$ respectively, under the initial prior $\pi_0(\phi)$. Then an iterative procedure is applied over $\pi_i(\phi)$ until the algorithm is converge, i.e., $R_U-R_L\leq\epsilon$. the output of the algorithm is the prior distribution $\pi(\phi)$ and the resultant regret can be evaluated also as $R_N^*(\Theta,\Phi) \approx \frac{(R_L+R_U)}{2}$.

    It is interesting to mention that the second divergence term of the algorithm (and the regret) $D_{c,N}(P_\phi \| \Theta)$ decreases exponentially the weight of the prior distribution $\pi(\phi)$ outside of the set $\Theta$, as implies by Theorem \ref{Thm2}.
    
    \begin{algorithm}        
    \caption{Arimoto-Blahut algorithm for Misspecified Batch Learning}\label{ArimotoBlahutAlg}
    \begin{algorithmic}
    \State \textbf{Input}: $N, \lambda, \epsilon, \Phi=\{\phi_m\}_{m=1}^{M_\phi}, \Theta=\{\theta_m\}_{m=1}^{M_\theta}, \pi_0(\phi)$
    \State \textbf{Output}: $\pi(\phi)$
    \State \textbf{Initialization}:
    \State $i \gets 0$
    \State $R_{L} = E_{\pi_0(\phi)}\left\{ D_{c,N}(P_\phi \| Q_{\pi_0}) -  D_{c,N}(P_\phi \| \Theta) \right\}$
    \State $R_{U} = \max_{\phi}\left( D_{c,N}(P_\phi \| Q_{\pi_0}) -  D_{c,N}(P_\phi \| \Theta) \right)$
    \State \textbf{Loop}:
    \begin{algorithmic}
    \While {$R_{U} - R_{L} > \epsilon$} 
            \State $i \gets i+1$
            \State $\tilde{\pi}(\phi_j)_{i+1} = \pi_i(\phi_j) \cdot 
            e^{\lambda \left(D_{c,N}(P_{\phi_j} \| Q_{\pi_i}) - D_{c,N}(P_{\phi_j} \| \Theta) \right)}$
            \State $\pi(\phi_j)_{i+1} = \frac{\tilde{\pi}(\phi_j)_{i+1}}{ \sum_{j'=0}^{M_{\phi}} \tilde{\pi}(\phi_{j'})_{i+1}}$
            \State $R_{L} = E_{\pi_{i+1}(\phi)}\left\{ D_{c,N}(P_\phi \| Q_{\pi_{i+1}}) -  D_{c,N}(P_\phi \| \Theta) \right\}$
            \State $R_{U} = \max_{\phi}\left( D_{c,N}(P_\phi \| Q_{\pi_0}) -  D_{c,N}(P_\phi \| \Theta) \right)$ 
    \EndWhile
    \State \textbf{end}
    \State \textbf{Return}: $\pi(\phi)$
    \end{algorithmic}
    \end{algorithmic}
    \end{algorithm}

    \subsection{Numerical Results}
    To demonstrate the Arimoto-Blahut algorithm extension we apply the Algorithm over the Bernoulli distributions where $y\in \{0,1\}$, $\Phi$ is the set of all the distributions such that $\phi \in [\phi_{min}, \phi_{max}]$, and the hypotheses set $\Theta$ is given by $\theta \in [a,b] \subset{\Phi}$, such that $0 \leq \phi_{\min} \leq a \leq b \leq \phi_{\max} \leq 1$.

    The numerical results of the regret in few interesting settings of $\Phi$ and $\Theta$ are summarized in Table \ref{NumericalResultsTable}. 
    
    \begin{table}[h!]
    \centering
     \begin{tabular}{||c c c c||} 
     \hline
     $\Phi$ & $\Theta$ & N & Regret \\ [0.5ex] 
     \hline\hline
     $[0,1]$ & $[\nicefrac{1}{4},\nicefrac{1}{2}]$ & $10^2$ & $\nicefrac{0.7242}{2N}$ \\ 

     $[\nicefrac{1}{4}-\delta,\nicefrac{3}{4}+\delta]$ & $[\nicefrac{1}{4}-\delta,\nicefrac{3}{4}+\delta]$ & $10^2$ & $\nicefrac{0.9171}{2N}$\\
     
     $[0,1]$ & $[\nicefrac{1}{4},\nicefrac{3}{4}]$ & $10^2$ & $\nicefrac{0.8728}{2N}$ \\ 
     
     $[\nicefrac{1}{4},\nicefrac{3}{4}]$ & $[\nicefrac{1}{4},\nicefrac{3}{4}]$ & $10^2$ & $\nicefrac{0.8710}{2N}$ \\     
     
     $[0,1]$ & $[\nicefrac{1}{3},\nicefrac{2}{3}]$ & $10^2$ & $\nicefrac{0.7869}{2N}$\\ 
     
     $[\nicefrac{1}{3},\nicefrac{2}{3}]$ & $[\nicefrac{1}{3},\nicefrac{2}{3}]$ & $10^2$ & $\nicefrac{0.7828}{2N}$\\      
     
     $[0,1]$ & $[\nicefrac{1}{100},\nicefrac{99}{100}]$ & $10^2$ & $\nicefrac{0.9766}{2N}$ \\
     
     $[\nicefrac{1}{100},\nicefrac{99}{100}]$ & $[\nicefrac{1}{100},\nicefrac{99}{100}]$ & $10^2$ & $\nicefrac{0.9763}{2N}$ \\ 
     
     $[0,1]$ & $[0,1]$ & $10^2$ & $\nicefrac{0.9908}{2N}$ \\
     
     $[0,1]$ & $[\nicefrac{1}{4},\nicefrac{1}{2}]$ & $10^3$ & $\nicefrac{0.9334}{2N}$ \\ 

     $[\nicefrac{1}{4}-\delta,\nicefrac{3}{4}+\delta]$ & $[\nicefrac{1}{4}-\delta,\nicefrac{3}{4}+\delta]$ & $10^3$ & $\nicefrac{0.9837}{2N}$\\
     
     $[0,1]$ & $[\nicefrac{1}{4},\nicefrac{3}{4}]$ & $10^3$ & $\nicefrac{0.9816}{2N}$\\ 
     
     $[\nicefrac{1}{4},\nicefrac{3}{4}]$ & $[\nicefrac{1}{4},\nicefrac{3}{4}]$ & $10^3$ & $\nicefrac{0.9798}{2N}$\\
     
     $[0,1]$ & $[\nicefrac{1}{100},\nicefrac{99}{100}]$ & $10^3$ & $\nicefrac{0.9970}{2N}$ \\
     
     $[\nicefrac{1}{100},\nicefrac{99}{100}]$ & $[\nicefrac{1}{100},\nicefrac{99}{100}]$ & $10^3$ & $\nicefrac{0.9970}{2N}$ \\
     
     $[0,1]$ & $[0,1]$ & $10^3$ & $\nicefrac{1.0027}{2N}$ \\ [1ex] 
     \hline
     \end{tabular}          
    \caption{Arimoto-Blahut Numerical Results Summary}
    \label{NumericalResultsTable}
    \end{table}

    \subsubsection{Stochastic vs. Misspecified Settings Example}
    Let us focus on the following three settings, where $N=10^3$:
    \begin{enumerate}[label=(\alph*)]
        \item Stochastic setting: $\tilde{\Phi}=\Theta = [\nicefrac{1}{4},\nicefrac{3}{4}]$, i.e., $a=\nicefrac{1}{4}$ and $b=\nicefrac{3}{4}$.
        \item Misspecified setting: $\Phi = [0,1]$ and $\Theta = [\nicefrac{1}{4},\nicefrac{3}{4}]$.
        \item Stochastic setting: $\tilde{\Phi}=\Theta_{\epsilon_N} = [\nicefrac{1}{4}-\delta,\nicefrac{3}{4}+\delta]$, where $\epsilon_N = {N^{\alpha-1}}$, $\alpha = 0.1$ and $\delta = \sqrt{2{a(1-a){\epsilon_N}}} = \sqrt{2{b(1-b){\epsilon_N}}} \approx 0.0274$. Note that $\Theta_{\epsilon_N} = \{ P_\phi \in \Phi : D_{c,N}( P_\phi, \Theta ) < \epsilon_N \}$ for $N=10^3$.
    \end{enumerate}
    The regrets of these three settings hold the following inequality as expected by Theorem \ref{Thm2}:
    \begin{align}
        \begin{aligned}\label{Thm2Demonstration}
            C_{c,N}(\Theta) \equiv \underbrace{R_N^*(\Theta)}_{\frac{0.9798}{2N}} < \underbrace{R_N^*(\Theta,\Phi)}_{\frac{0.9816}{2N}} < \underbrace{R_N^*(\Theta_{\epsilon_N})}_{\frac{0.9837}{2N}} \equiv C_{c,N}(\Theta_{\epsilon_N})
        \end{aligned}
    \end{align}
    In addition, we get numerically that $R_N^*(\Theta,\Phi) - C_{c,N}(\Theta) \approx \frac{0.0018}{2N}$ (both for $N=10^2$ and $N=10^3$).
    This fact and (\ref{Thm2Demonstration}) implies that there is an $\epsilon_N \to 0$ such that $C_{c,N}(\Theta_{\epsilon_N}) \to C_{c,N}(\Theta)$ and $R_N^*(\Theta,\Phi) \approx C_{c,N}(\Theta)$.
    Another aspect that demonstrates this result is the similarity between the capacity achieving prior distribution $\pi(\phi)$ of these examples, as can be seen in Figure \ref{FigurePriorTheta_025_075_N1000}. Note that the prior distribution for the setting $\tilde{\Phi} = \Theta$ is zeroed outside the region of $\Theta = [\nicefrac{1}{4},\nicefrac{3}{4}]$, while it decays rapidly outside this region as expected where $\Phi = [0,1]$.

    \begin{figure}[h]
        \centering
        \includegraphics[width=0.5\textwidth]{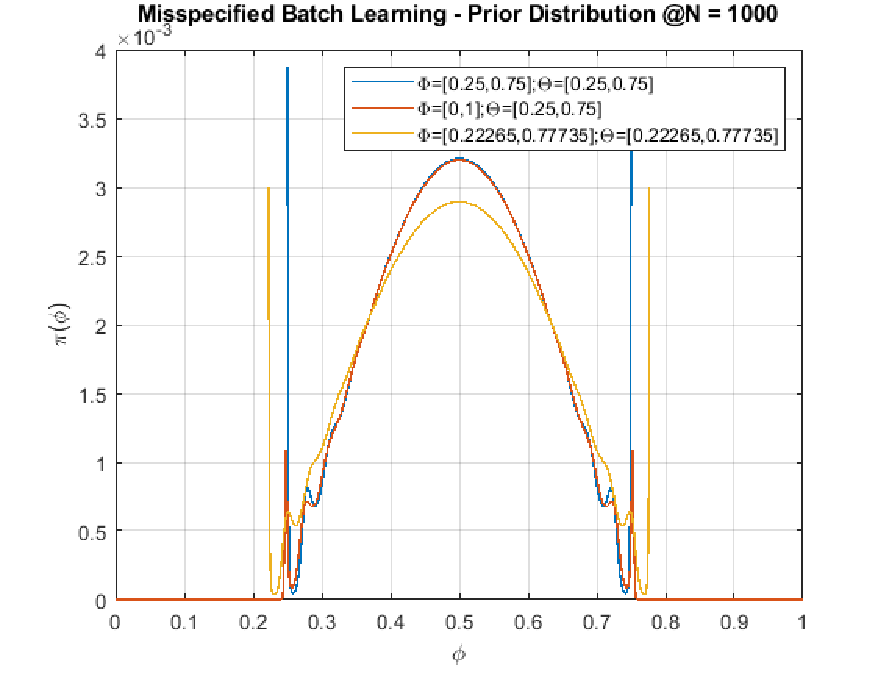}
        \caption{Stochastic vs. Misspecified Capacity Achieving Prior distribution for settings: (a), (b) and (c), where $N=10^3.$}
        \label{FigurePriorTheta_025_075_N1000}
    \end{figure}

    A similar result for $N=10^2$ and $\delta = \sqrt{2{a(1-a){\epsilon_N}}} = \sqrt{2{b(1-b){\epsilon_N}}} \approx 0.0771$ can be found in Table \ref{NumericalResultsTable}.
    
    \subsubsection{Add-$\beta$ Factor Analysis}\label{Bias Factor Analysis}
    Another interesting term to be analyzed is the add-$\beta$ factor as function of the empirical distribution $P_{emp} \equiv \frac{\sum_{n=1}^{N-1}{y_n}}{N-1}$ of the $N-1$ training samples. Note that $P_{emp}$ is the sufficient statistic of the universal predictor $Q_{\pi}(y_N=1 | y^{N-1})$. Hence, we can get the following:
    \begin{align}
        \begin{aligned}
            Q_{\pi}(y_N=1 | y^{N-1}) &= Q_{\pi}\left(y_N=1 | \sum_{n=1}^{N-1}{y_n}\right) 
            \\&= Q_{\pi}\left(y_N=1 | P_{emp}\right)  
            \\&= \frac{P_{emp}(N-1) + \beta}{N-1+2\beta}
        \end{aligned}
    \end{align}
    and by simple algebraic manipulations we get:
    \begin{align}
        \begin{aligned}
            \beta(P_{emp}) = (N-1)\frac{Q_{\pi}(y_N=1 | P_{emp}) - P_{emp}}{1-2Q_{\pi}(y_N=1 | P_{emp})}.
        \end{aligned}
    \end{align}    
    In Figure \ref{FigureBeta_Theta_001_099_N100} we demonstrate the bias factor $\beta(P_{emp})$ in the following settings:
    \begin{enumerate}[label=(\alph*)]
        \item Misspecified setting, where $\Theta = [\nicefrac{1}{100},\nicefrac{99}{100}]$ and $\Phi = [0,1]$.
        \item Stochastic setting, where $\Phi=\Theta=[\nicefrac{1}{100},\nicefrac{99}{100}]$.
        \item Stochastic setting, where $\Phi=\Theta=[0,1]$.
    \end{enumerate}
    \begin{figure}[h]
        \centering
        \includegraphics[width=0.5\textwidth]{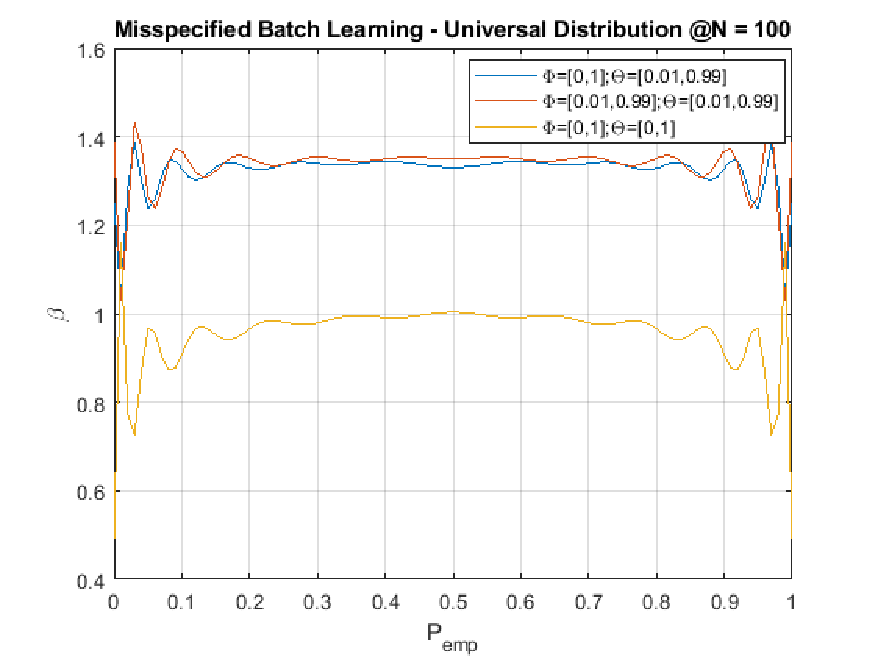}
        \caption{{Stochastic vs. Misspecified $\beta$ bias factor for settings: (a), (b)  and (c), where $N=10^2$}.}
        \label{FigureBeta_Theta_001_099_N100}
    \end{figure}

    It can be seen that settings (a) and (b) have a similar add-$\beta$ factor that toggles approximately around $\sim1.3$, while in settings (c) the add-$\beta$ factor is lower than (a) and (b) settings and toggles approximately around $\sim1$, i.e., the add-$\beta$ factor is governed by the set $\Theta$.
    Note that in the stochastic online setting, where $\Phi=[0,1]$, the add-$\beta$ is a constant factor and is given asymptotically by $\nicefrac{1}{2}$, see \cite{UniversalPrediction}. 
    
    Another interesting comparison is the extreme case where $P_{emp}=0$. In the misspecified setting (a), $\beta(a) \approx 1.25$, and in the stochastic setting (b) and (c), $\beta(b) \approx 1.38$ and $\beta(c) \approx 0.49$, respectively.
    
    \section{Misspecified Universal Batch Learning Extensions}\label{Misspecified Universal Batch Learning Extensions}
    \subsection{Combined Batch and Online Learning}    
    Let us extend the basic problem statement to the setting of universal prediction of $L$ data outcomes, denoted by $y^L = y_{N+1},y_{N+2},\dots,y^{N+L}$, given $N$ training data samples, denoted by $y^N$, under the misspecification setting. Note that a similar extension in the stochastic setting was analyzed in \cite{FogelFederCalculation}.
    
    For doing so, let us define the min-max regret as follows:
    \begin{align}
        \begin{aligned}
            R_{N,L}^*(\Theta,\Phi) = \min_{Q} \max_{P_{\phi} \in \Phi} \max_{P_{\theta} \in \Theta} \left( \frac{1}{L} \sum_{y^{N+L}}P_{\phi}(y^{N+L})\log \left({\frac{P_{\theta}(y^{L}|y^{N})}{Q(y^L|y^{N})}}\right) \right),
        \end{aligned}
    \end{align}
    and the conditional divergence between a probability distribution $P_\phi \in \Phi$ to the set of hypotheses $\Theta$ by the following:
    \begin{align}
        \begin{aligned}
            D_{c,N,L}(P_{\phi} \| \Theta) \equiv \min_{P_{\theta} \in \Theta}\underbrace{D \left( P_{\phi}\left(Y^{L} | Y^N\right) \| P_{\theta}\left( Y^{L} | Y^N \right) \right)}_{D_{c,N,L}(P_\phi \| P_\theta)}
        \end{aligned}
    \end{align}

    \subsubsection{Min-Max Regret Derivation}
    Following the same steps as in the misspecified batch learning setting derivations, and the above definitions we can get the following Theorem:
    \begin{thm}\label{thmCombinedSetting}
       The min-max regret of the problem defined above is given by:		
        \begin{align}
            \begin{aligned}
                R_{N,L}^*(\Theta,\Phi) = \max_{\pi(\phi)} \left( \frac{1}{L} \left( I\left(Y^L ; \Phi | Y^N\right) -E_{\pi(\phi)}\left\{ D_{c,N,L}(P_{\phi} \| \Theta) \right\} \right) \right)
            \end{aligned}
        \end{align}
        and the universal distribution for a given $\pi(\phi)$ is given by:
        \begin{align}
            \begin{aligned}\label{CombinedUniversalDistribution}
                Q_{\pi}\left( y^{L} | y^{N} \right) = \frac{\int{\pi(\phi)P_{\phi}(y^{N+L})d\phi}}{\int{\pi(\phi)P_{\phi}(y^{N})d\phi}}.
            \end{aligned}
        \end{align} 
    \end{thm}
    The form of the min-max regret, as given by Theorem \ref{thmCombinedSetting} for the above problem setting, is similar to the regret of the combined batch and online prediction under the stochastic setting, see \cite{FogelFederCalculation}. Clearly, for $L=1$, the result coincides with Theorem \ref{ThmFirst}, i.e., the min-max regret of the misspecified batch learning problem.

    \subsubsection{Bounding the Min-Max Regret}
    Now we turn to derive upper and lower bounds for the misspecified combined batch and online regret. The bounds can be interpreted as a generalization of Theorem \ref{Thm2} and Theorem 4 in \cite{SequentialMisspecification} bounds for the misspecified batch and misspecified online learning regrets respectively.
    \begin{thm}\label{thmCombinedSettingBounds}
       Suppose $\Theta \subseteq \Phi$ are sets of distributions s.t. the data samples are i.i.d and $C_{c,M}(\Phi) \equiv \tau_M \to 0$. Then for every set of $\epsilon_{N+t} \gg \tau_{N+t}$ for $t=1,2,\dots,L$ we have
        \begin{align}
            \begin{aligned}
                 R_{N,L}^*(\Theta) \leq R_{N,L}^*(\Theta,\Phi) \leq \frac{1}{L}\sum_{t=1}^LC_{c,N+t}\left( \Theta_{\epsilon_{N+t}} \right) + o(1)
            \end{aligned}
        \end{align}
        where $R_{N,L}^*(\Theta)$ is the combined batch and online regret under the stochastic setting and $\Theta_{\epsilon} \equiv \{P_{\phi} \in \Phi: D_{c}(P_\phi||\Theta) < \epsilon \}$.
    \end{thm}
    \begin{proof}
        The basis of the proof is given by applying the chain rule to the universal distribution, defined in (\ref{CombinedUniversalDistribution}), by the following:
        \begin{align}
            \begin{aligned}
                Q_\pi(y^L | y^N) &= Q_\pi(y_{N+1} | y^{N}) \cdot Q_\pi(y_{N+2} | y^{N+1}) \cdots Q_\pi(y_{N+t} | y^{N+t-1}) \cdots Q_\pi(y_{N+L} | y^{N+L-1}) 
                \\&= \Pi_{t=1}^L Q_\pi(y_{N+t} | y^{N+t-1}) = \Pi_{t=1}^L \frac{\int{\pi(\phi)P_{\phi}(y^{N+t})d\phi}}{\int{\pi(\phi)P_{\phi}(y^{N+t-1})d\phi}}.
            \end{aligned}
        \end{align}
        Clearly, the regret holds the following:
        \begin{align}
            \begin{aligned}
                R_{N,L}^*(\Theta,\Phi) = \max_{\pi(\phi)} \frac{1}{L} \int{ \pi(\phi) \sum_{y^{N+L}} P_\phi(y^{N+L}) \log \left( \frac{P_\theta(y^L | y^N)}{Q_\pi(y^L | y^N)} \right) }d\phi,
            \end{aligned}
        \end{align}
        and according to the Theorem assumptions the set of hypotheses assumes data samples are i.i.d, i.e., $P_\theta(y^L | y^N) = \Pi_{t=1}^L P_\theta(y_{N+t})$. Therefore, by combining all the above we get the following:
        \begin{align}
            \begin{aligned}
                R_{N,L}^*(\Theta,\Phi) &= \max_{\pi(\phi)} \frac{1}{L} \int{ \pi(\phi) \sum_{y^{N+L}} P_\phi(y^{N+L}) \log \left( \frac{\Pi_{t=1}^LP_\theta(y_{N+t})}{\Pi_{t=1}^LQ_\pi(y_{N+t} | y^{N+t-1})} \right) }d\phi
                \\& = \max_{\pi(\phi)} \frac{1}{L} \sum_{t=1}^L \int{ \pi(\phi) \sum_{y^{N+t}} P_\phi(y^{N+t}) \log \left( \frac{P_\theta(y_{N+t})}{Q_\pi(y_{N+t} | y^{N+t-1})} \right) }d\phi
                \\& \leq \frac{1}{L} \sum_{t=1}^L \max_{\pi(\phi)} \int{ \pi(\phi) \sum_{y^{N+t}} P_\phi(y^{N+t}) \log \left( \frac{P_\theta(y_{N+t})}{Q_\pi(y_{N+t} | y^{N+t-1})} \right) }d\phi 
                \\& \leq \frac{1}{L} \sum_{t=1}^L R^*_{N+t}(\Theta,\Phi) 
                \leq \frac{1}{L}\sum_{t=1}^LC_{c,N+t}\left( \Theta_{\epsilon_{N+t}} \right) + o(1).
            \end{aligned}
        \end{align}
        The first inequality follows trivially due to the fact that summation over maximized terms is greater than the maximization of the summed terms. The second inequality is given by the misspecified batch learning definition, and the last inequality is given by Theorem \ref{Thm2}.

        The lower bound of the misspecified combined batch and online regret is given trivially by definition, i.e., by taking $\Phi \equiv \Theta$, which gives us the combined batch and online regret of the stochastic setting.
    \end{proof}

    Assuming a set of hypotheses $\Theta$ such that $\epsilon_N \to 0$ and $\pi_N(\theta) \to \pi^*(\theta)$, i.e., the capacity achieving prior distribution is not a function of $N$ for large enough $N$, then the following holds: 
    \begin{align}
        \begin{aligned}
            \frac{1}{L}\sum_{t=1}^LC_{c,N+t}\left( \Theta_{\epsilon_{N+t}} \right) &= \frac{1}{L}\sum_{t=1}^L \max_{\pi^*(\theta)} I(Y_{N+t};\Theta|Y^{N+t-1}) + o(1)
            \\&= \max_{\pi^*(\theta)} \frac{1}{L} \sum_{t=1}^L I(Y_{N+t};\Theta|Y^{N+t-1}) + o(1)
            \\&= \max_{\pi^*(\theta)} \frac{1}{L} I(Y^{L};\Theta|Y^{N}) + o(1)
            \\&= R_{N,L}^*(\Theta) + o(1).
        \end{aligned}
    \end{align}    
    Summarizing all the above, we get the following:
    \begin{align}
        \begin{aligned}
            \label{combinedMissRegert}
            R_{N,L}^*(\Theta) \leq R_{N,L}^*(\Theta,\Phi) \leq R_{N,L}^*(\Theta) + o(1).
        \end{aligned}
    \end{align}
    Namely, the misspecified combined batch and online learning regret equals approximately to the regret of the batch and online learning under the stochastic setting over the set of hypotheses $\Theta$.
    This result generalizes the result of Theorem \ref{Thm2} to the misspecified combine batch and online learning problem.
    Note that in \cite{FogelFederCalculation} the combined batch and online learning under the stochastic setting over the multinomial of order $m$ set of hypotheses was analyzed. It has been shown that for $N \gg L \gg 1$ the combined regret converges to the regret of the batch learning and $R_{N,L}(\Theta)^* \approx \frac{m-1}{2N}$. In addition, in the opposite case where $L \gg N \gg 1$ it has been shown that the combined regret converges to the online learning regret and $R_{N,L}(\Theta)^* \approx \frac{m-1}{2L}\log(L)$.
    According to (\ref{combinedMissRegert}) these extreme cases also hold in the misspecified combined batch and online learning setting.

    \subsubsection{Arimoto Blahut Algorithm Extension}
    Following the same steps as in section (\ref{Arimoto Blahut Algorithm Extension}) for the Arimoto-Blahut algorithm extension development of the misspecified batch learning problem we can extend the algorithm to the the misspecified combined batch and online learning setting.
    The algorithm stopping rule is based on the following upper and lower bounds:
    \begin{cor}
        The min-max regret of the combined batch and online learning under the misspecification setting holds the following for any $\Phi$, $\Theta$ and $\pi(\phi)$:
        \begin{align}
            \begin{aligned}
                R_L \leq R^*_{N,L} \leq R_U
            \end{aligned}
        \end{align}
        where,
        \begin{align}
            \begin{aligned}
                R_L \equiv E_{\pi(\phi)}\{ D_{c,N,L}(P_\phi \| Q_\pi) - D_{c,N,L}(P_\phi \| \Theta) \}
            \end{aligned}
        \end{align}
        and
        \begin{align}
            \begin{aligned}
                R_U \equiv \max_{P_\phi} ( D_{c,N,L}(P_\phi \| Q_\pi) - D_{c,N,L}(P_\phi \| \Theta) ).
            \end{aligned}
        \end{align}
    \end{cor}
    \begin{proof}
        Similar to the proof of Theorem \ref{thmArimotoBlahut}.
    \end{proof}
    The pseudo-code of the algorithm is given in \textbf{Algorithm} \textbf{\ref{ArimotoBlahutAlgComb}}.
    
    \begin{algorithm}        
    \caption{Arimoto-Blahut alg. for Misspecified Batch and Online Learning}\label{ArimotoBlahutAlgComb}
    \begin{algorithmic}
    \State \textbf{Input}: $N, L, \lambda, \epsilon, \Phi=\{\phi_m\}_{m=1}^{M_\phi}, \Theta=\{\theta_m\}_{m=1}^{M_\theta}, \pi_0(\phi)$
    \State \textbf{Output}: $\pi(\phi)$
    \State \textbf{Initialization}:
    \State $i \gets 0$
    \State $R_{L} = E_{\pi_0(\phi)}\left\{ D_{c,N,L}(P_\phi \| Q_{\pi_0}) -  D_{c,N,L}(P_\phi \| \Theta) \right\}$
    \State $R_{U} = \max_{\phi}\left( D_{c,N,L}(P_\phi \| Q_{\pi_0}) -  D_{c,N,L}(P_\phi \| \Theta) \right)$
    \State \textbf{Loop}:
    \begin{algorithmic}
    \While {$R_{U} - R_{L} > \epsilon$} 
            \State $i \gets i+1$
            \State $\tilde{\pi}(\phi_j)_{i+1} = \pi_i(\phi_j) \cdot 
            e^{\lambda \left(D_{c,N,L}(P_{\phi_j} \| Q_{\pi_i}) - D_{c,N,L}(P_{\phi_j} \| \Theta) \right)}$
            \State $\pi(\phi_j)_{i+1} = \frac{\tilde{\pi}(\phi_j)_{i+1}}{ \sum_{j'=0}^{M_{\phi}} \tilde{\pi}(\phi_{j'})_{i+1}}$
            \State $R_{L} = E_{\pi_{i+1}(\phi)}\left\{ D_{c,N,L}(P_\phi \| Q_{\pi_{i+1}}) -  D_{c,N,L}(P_\phi \| \Theta) \right\}$
            \State $R_{U} = \max_{\phi}\left( D_{c,N,L}(P_\phi \| Q_{\pi_0}) -  D_{c,N,L}(P_\phi \| \Theta) \right)$ 
    \EndWhile
    \State \textbf{end}
    \State \textbf{Return}: $\pi(\phi)$
    \end{algorithmic}
    \end{algorithmic}
    \end{algorithm}
    
    \subsection{Supervised Batch Learning}
    In this section we turn to extend our main results to the setting of the supervised data setting. In this setting, for each data feature or side information, denoted by $x_N$, there is a labeled data sample, denoted by $y_N$. The problem statement is to assign a universal conditional distribution to the next data label $y_N$ given its appropriate data feature $x_N$ and the training set $x^{N-1}$ and $y^{N-1}$. We denote this universal distribution by $Q(y_N | x^N, y^{N-1})$.
    We assume that the data features are i.i.d, i.e., $P(x^N)=\Pi_{n=1}^NP(x_n)$, where $P(x)$ is a known distribution and the conditional probability between $y_N$ to $x_N$ is a Discrete Memoryless Channel (DMC) like, i.e., $P_\phi(y^N|x^N)=\Pi_{n=1}^NP_\phi(y_n|x_n)$ for any $P_\phi \in \Phi$. 
    
    The min-max regret of the above supervised batch learning problem under the misspecified setting is defined as follows:
    \begin{align}
        \begin{aligned}
            R_N^*(\Theta,\Phi) = \min_{Q} \max_{P_\theta \in \Theta} \max_{P_\phi \in \Phi} R_N\left(P_\phi,P_\theta,Q\right)
        \end{aligned}
    \end{align}
    where,
    \begin{align}
        \begin{aligned}
            R(P_\phi,P_\theta,Q) &= \sum_{x^N} \sum_{y^N} {P(x^N)}{P_{\phi}(y^N|x^N)} \log\left( \frac{P_{\theta}(y_N|x_N)}{Q(y_N|x^N,y^{N-1})} \right)
            \\& = \sum_{x^N} \sum_{y^N} {P(x^N)}{P_{\phi}(y^N|x^N)} \log\left( \frac{P_{\phi}(y_N|x_N)}{Q(y_N|x^N,y^{N-1})} \right) 
            \\&- D_{c}(P_\phi \| P_\theta).
        \end{aligned}
    \end{align}
    and
    \begin{align}
        \begin{aligned}
            D_{c}(P_\phi \| P_\theta) = \sum_{x_N}\sum_{y_N} P(x_N)P_\phi(y_N) \log\left( \frac{P_\phi(y_N | x_N)}{P_\theta(y_N | x_N)} \right).
        \end{aligned}
    \end{align}
    In addition, let us define the conditional divergence between a probability distribution $P_\phi \in \Phi$ to the set of hypotheses $\Theta$ by the following:
    \begin{align}
        \begin{aligned}
            D_{c}( P_\phi \| \Theta ) \equiv \min_{P_\theta \in \Theta} D_c\left( P_\phi \| P_\theta \right).
        \end{aligned}
    \end{align}

    Under the above definitions, the min-max regret of the misspecified supervised batch learning setting is given by the following Theorem:
    \begin{thm}
        The min-max regret of the problem defined above is given by:		
		\begin{align}									
			\begin{aligned}		
				\label{thmRegret}	
			R^*_N(\Theta,\Phi) = \max_{\pi({\phi})}  \left( I(Y_N;\Phi|X^N,Y^{N-1}) - E_{\pi(\phi)} \{  D_{c}({{P_{\phi}}\|\Theta}) \} \right)
		\end{aligned}
		\end{align}
		and the universal distribution for a given $\pi(\phi)$ is given by:
		\begin{align}
			\begin{aligned}
				Q_{\pi}({y_N|x^N,y^{N-1}}) = \frac{\int_{\phi}\pi(\phi)P_{\phi}(y^{N}|x^N)  \,d\phi}{\int_{\phi}\pi(\phi)P_{\phi}(y^{N-1}|x^{N-1})  \,d\phi}.
			\end{aligned}
		\end{align}
    \end{thm}

    \begin{proof}        
    \begin{align}
        \begin{aligned}
            R_N^*(\Theta,\Phi) &= \min_{Q} \max_{P_\theta \in \Theta} \max_{P_\phi \in \Phi} R(P_\phi,P_\theta,Q)             
            \\&= \min_{Q} \max_{P_\phi \in \Phi} \Bigg( \sum_{x^N} \sum_{y^N} {P(x^N)}{P_{\phi}(y^N|x^N)} \log\left( \frac{P_{\phi}(y_N|x_N)}{Q(y_N|x^N,y^{N-1})} \right)
            \\&- \min_{P_\theta \in \Theta}D_{c}(P_\phi \| P_\theta) \Bigg)            
            \\&\equiv \min_{Q}\max_{\pi(\phi)}R(\pi(\phi),Q)
        \end{aligned}
    \end{align}
    where,
    \begin{align}
        \begin{aligned}
            R(\pi(\phi),Q) &\equiv \int_{\phi}{\pi(\phi) \sum_{x^N} \sum_{y^N} {P(x^N)}{P_{\phi}(y^N|x^N)} \log\left( \frac{P_{\phi}(y_N|x_N)}{Q(y_N|x^N,y^{N-1})} \right) }d\phi
            \\& - \int_{\phi}{\pi(\phi) \min_{P_\theta \in \Theta}D_{c}(P_\phi \| P_\theta)}d\phi
            \\&= \int_{\phi}{\pi(\phi) \sum_{x^N} \sum_{y^N} {P(x^N)}{P_{\phi}(y^N|x^N)} \log\left( \frac{P_{\phi}(y_N|x_N)}{Q(y_N|x^N,y^{N-1})} \right) }d\phi 
            \\& - E_{\pi(\phi)}\left\{ \min_{P_\theta \in \Theta} D_{c}( P_\phi \| P_\theta ) \right\}.
        \end{aligned}
    \end{align}
    Using the min-max theorem we get:
    \begin{align}
        \begin{aligned}
            R_N^*(\Theta,\Phi) &= \min_{Q}\max_{\pi(\phi)}R(\pi(\phi),Q) 
            \\&= \max_{\pi(\phi)}\min_{Q}R(\pi(\phi),Q).
        \end{aligned}
    \end{align}    
    
    Defining the following Lagrangian:    
    \begin{align}
        \begin{aligned}
            L = R(\pi(\phi),Q) + \sum_{x^N}\sum_{y^{N-1}}\lambda_{x^N,y^{N-1}}\sum_{y_N}Q(y_N | x^N, y^{N-1})
        \end{aligned}
    \end{align}
    and zeroing the derivative of $L$ w.r.t $Q$ we get:    
    \begin{align}
        \begin{aligned}
            \frac{\partial L}{\partial Q} = -\frac{1}{Q(y_N | x^N, y^{N-1})} \int_{\phi}{\pi(\phi) \Pi_{n=1}^NP(x_n) \Pi_{n=1}^NP_{\phi}(y_n|x_n)}d\phi + \lambda_{x^N,y^{N-1}} = 0
        \end{aligned}
    \end{align}
    and to hold the constraint $\sum_{y_N}Q(y_N | x^N, y^{N-1}) = 1$ we get:
    \begin{align}
        \begin{aligned}
            \lambda_{x^N,y^{N-1}} = \int_{\phi}{\pi(\phi) \Pi_{n=1}^NP(x_n) \Pi_{n=1}^{N-1}P_{\phi}(y_n|x_n)}d\phi.
        \end{aligned}
    \end{align}    
    Combining all the above, gives us:
    \begin{align}
        \begin{aligned}
            {Q(y_N | x^N, y^{N-1})} &= \frac{\int_{\phi}{\pi(\phi) \Pi_{n=1}^NP_{\phi}(y_n|x_n)}d\phi}{\int_{\phi}{\pi(\phi) \Pi_{n=1}^{N-1}P_{\phi}(y_n|x_n)}d\phi}
            \\& = \frac{\int_{\phi}{\pi(\phi) P_{\phi}(y^N|x^N)}d\phi}{\int_{\phi}{\pi(\phi) P_{\phi}(y^{N-1}|y^{N-1})}d\phi}.
        \end{aligned}
    \end{align}    
    Let us define the joint probability distribution function of $\Phi$, $X^N$ and $Y^N$ by:
    \begin{align}
        \begin{aligned}
            P(\phi,x^N,y^N) = \pi(\phi)\Pi_{n=1}^NP(x_n)\Pi_{n=1}^NP_{\phi}(y_n|x_n)     
        \end{aligned}
    \end{align}    
    then it can be verified that the following holds:
    \begin{align}
        \begin{aligned}
            &P(y_N|x^N,y^{N-1}) = Q(y_N | x^N, y^{N-1})
            \\&
            P(y_N|\phi,x^N,y^{N-1}) = P_\phi(y_N|x_N)
        \end{aligned}
    \end{align}
    and therefore we get:
    \begin{align}
        \begin{aligned}
            R_N^*(\Theta,\Phi) &= \max_{\pi(\phi)} \Bigg( \int_{\phi}{\pi(\phi) \sum_{x^N} \sum_{y^N} P(\phi,x^N,y^N) \log\left( \frac{P(y_N|\phi,x^N,y^{N-1})}{P(y_N|x^N,y^{N-1})} \right) }d\phi 
            \\&- E_{\pi(\phi)}\left\{ \min_{P_\theta \in \Theta} D_{c}( P_\phi |\ P_\theta ) \right\} \Bigg)
            \\& = \max_{\pi(\phi)} \left( I\left( Y_N;\Phi | X^N, Y^{N-1} \right) 
            - E_{\pi(\phi)}\left\{ D_{c}( P_\phi \| \Theta ) \right\} \right).
        \end{aligned}
    \end{align}    
    \end{proof}    

    Note that the general case, where the data features distribution is unknown, and in addition, our conjecture that a similar result to Theorem \ref{Thm2} also holds in this supervised setting is still under investigation.
    
    \section{Conclusions and Future Research Directions}\label{Conclusions and Future Research Directions}
    In this paper we discussed the universal batch learning problem under the misspecification setting with log-loss. We introduced the universal probability distribution for the next data       outcome given the batch training samples as a mixture over the set of data generating distribution set. Moreover, we derived closed form expression for the min-max regret and presented it by information theoretical tool, as a constrained version of the conditional capacity between the data and the data generating distribution set. Furthermore, we derived tight bounds for the min-max regret, and showed that it approximately performs as in the problem of batch learning under the stochastic setting. Surprisingly, this means that the complexity of the problem is governed by the richness of the hypotheses set and not by the larger set of data generating distributions. We demonstrated this observation in the example where $\Phi$ is the set of all Bernoulli distributions and showed that its regret equals approximately to the conditional capacity of only a subset of Bernoulli distributions defined by the set of hypotheses $\Theta$.
    In addition, we developed an extension to the Arimoto-Blahut algorithm for numerical evaluation of the capacity-like achieving prior distribution and the regret. We applied this algorithm over few interesting Bernoulli sets of $\Phi$ and $\Theta$.
    Using the techniques and theoretical tools we developed during the analysis of the unsupervised batch learning setting we extended our analysis to the combined batch and online setting and to the supervised batch learning under the misspecification setting.

    This work raises several topics for further research. In this paper we concentrated on the unsupervised universal batch learning under the misspecification setting problem. 
    While the unsupervised online learning under the misspecfication setting has been analyzed in \cite{UniversalPrediction},\cite{RobustInference} the corresponding supervised setting is still missing and needs to be analyzed. In addition, the result of Theorem \ref{Thm2} does not guarantee that always $C_{c,N}(\Theta_{\epsilon_N}) \to C_{c,N}(\Theta)$. Further analysis should be done to find the necessary and sufficient conditions to meet this.

    Finally, this work and these additional challenges combined with many other referenced works here are an additional part of the quest to establish information theory approach in the theoretical framework of statistical machine learning.


\begin{thebibliography}{9}
        \bibitem{DistributionEstimation}
        S.~Kamath, A.~Orlitsky, D.~Pichapati, and A.~T.~Suresh, “On learning
        distributions from their samples,” \emph{in Conference on Learning Theory},
        2015, pp. 1066–1100.
        
        \bibitem{StatisticalLearning1}
        O.~Bousquet, S.~Boucheron, and G.~Lugosi, “Introduction to statistical learning theory,” 
        \emph{in Advanced lectures on machine learning}. Springer, 2004, pp. 169–207.

        \bibitem{StatisticalLearning2}
        V.~N.~Vapnik, “An overview of statistical learning theory,” 
        \emph{IEEE transactions on neural networks}, vol. 10, no. 5, pp. 988–999, 1999.

        \bibitem{UniversalPrediction}
		N.~Merhav and M.~Feder, “Universal prediction,” \emph{IEEE Transactions on
		Information Theory}, vol. 44, no. 6, pp. 2124–2147, 1998.

        \bibitem{RobustInference}
		A.~Painsky and M.~Feder, 
		"Robust Universal Inference,"
		\emph{Entropy}, 23(6), 773, 2021.

        \bibitem{SequentialMisspecification}
		M.~Feder and Y.~Polyanskiy,
		"Sequential prediction under log-loss and misspecification,"
		\emph{Proceedings of Machine Learning Research}, vol. 134, pp. 1–28, 2021.        

        \bibitem{OnlineLearning}
		Y.~Fogel and M.~Feder, "On the Problem of On-line Learning with Log-Loss",
		\emph{IEEE International Symposium on Information Theory (ISIT)}, 2017.
  
		\bibitem{BatchLearning}
		Y.~Fogel and M.~Feder,
		"Universal Batch Learning with Log-Loss,"
		\emph{IEEE International Symposium on Information Theory (ISIT)}, 2018.

        \bibitem{UniversalLearningIndividual}
		Y.~Fogel and M.~Feder, "Universal Supervised Learning for Individual Data", 
		\emph{Draft}, https://doi.org/10.48550/arXiv.1812.09520, 2018.
		
        \bibitem{UniversalLearningIndividualISIT}
        Y.~Fogel and M.~Feder, "Universal Learning of Individual Data",
        \emph{IEEE International Symposium on Information Theory (ISIT)}, 2019.
		
        \bibitem{DeepPNML}
        K.~Bibas, Y.~Fogel and M.~Feder, "Deep pNML: Predictive Normalized Maximum Likelihood for Deep Neural Networks"; \emph{Draft}, https://doi.org/10.48550/arXiv.1904.12286, 2020.

        \bibitem{Barron}
        J.~Takeuchi and A.~R.~Barron, "Robustly Minimax Codes for Universal Data Compression", 1998.

        \bibitem{NirPaper}
        N.~Weinberger and M.~Feder, "On Information-Theoretic Determination of Misspecified Rates of Convergence", \emph{IEEE International Symposium on Information Theory (ISIT)}, 2022.
        
        \bibitem{Blahut}
        R.~Blahut, "Computation of channel capacity and rate-distortion functions",\emph{ IEEE Trans. Inf. Theory}, 1972, 18, 460–473.
        
        \bibitem{Arimoto}
        S.~Arimoto, "An algorithm for computing the capacity of arbitrary discrete memoryless channels", \emph{IEEE Trans. Inf. Theory}, 1972,
        18, 14–20.

        \bibitem{FogelFederCalculation}
		Y.~Fogel and M.~Feder,
		"Combining Batch and Online Prediction,"
        \emph{Learn to Compress workshop, ISIT}, 2024.

        \bibitem{ArimotoBlahutExt1}
        G.~Matz and P.~Duhamel, “Information geometric formulation and interpretation of accelerated blahut-arimoto-type algorithms”, \emph{in Information
        theory workshop. IEEE}, 2004, pp. 66–70.

        \bibitem{ArimotoBlahutExt2} H.~H.~Mattingly, M.~K.~Transtrum, M.~C.~Abbott, and B.~B.~Machta,
        “Maximizing the information learned from finite data selects a simple model”, \emph{Proceedings of the National Academy of Sciences}, vol. 115,
        no. 8, pp. 1760–1765, 2018.

        \bibitem{ArimotoBlahutExt3} I.~Naiss and H.~H.~Permuter,
        “Extension of the Blahut-Arimoto Algorithm for Maximizing Directed Information”, \emph{IEEE Trans. Inf. Theory}, vol. 59,
        no. 1, pp. 204–222, 2013.

        \bibitem{ArimotoBlahutExt4} M.~C.~Abbott and B.~B.~Machta, “A scaling law from discrete to continuous solutions of channel capacity problems in the low-noise
        limit”, \emph{Journal of Statistical Physics}, vol. 176, pp. 214–227, 2019.
        \bibitem{parametricExample}
        D.~Braessa and T.~Sauer, "Bernstein polynomials and learning theory", \emph{Journal of Approximation Theory}, vol. 128, pp. 187–206, 2004.
        
        \bibitem{Sion}
        M.~Sion, "On general minimax theorems", \emph{Pac. J. Math.}, 1958, 8, 171–176.


  
	\end{thebibliography}
\end{document}